\documentclass[conference,a4paper]{IEEEtran}
\IEEEoverridecommandlockouts

\usepackage[T1]{fontenc}
\usepackage[utf8]{inputenc}
\usepackage{microtype}
\usepackage{amsmath,amssymb,amsthm}
\newtheorem{lemma}{Lemma}
\newtheorem{corollary}{Corollary}
\newtheorem{remark}{Remark}

\usepackage{graphicx}
\usepackage{stfloats}

\usepackage{booktabs}
\usepackage{cite}
\usepackage{url}
\usepackage[colorlinks=true,citecolor=blue,linkcolor=black,urlcolor=blue]{hyperref}
\hypersetup{pdftitle={Valid for Free: Homophily-Gated Conformal Prediction for Training-Free Node Classification with Tabular Foundation Models}}

\begin{document}

\title{Valid for Free: Homophily-Gated Conformal Prediction for Training-Free Node Classification with Tabular Foundation Models}

\author{
\IEEEauthorblockN{Nguyen Duy Long\textsuperscript{1}, Phung Minh Hien\textsuperscript{1},
Nguyen Trong Viet\textsuperscript{2}, and Nguyen Thai Anh\textsuperscript{3,*}}
\IEEEauthorblockA{\textsuperscript{1}Faculty of Information Technology,
Posts and Telecommunications Institute of Technology, Hanoi, Vietnam\\
\textsuperscript{2}Faculty of Engineering, School of Computer Science,
The University of Sydney, Sydney, Australia\\
\textsuperscript{3}Faculty of Information Technology, Van Lang School of
Technology, Van Lang University, Ho Chi Minh City, Vietnam\\
Email: longnd.b24cn364@stu.ptit.edu.vn, hienpm.b24cn199@stu.ptit.edu.vn, trng0882@uni.sydney.edu.au,
anh.nt@vlu.edu.vn\\
\textsuperscript{*}Corresponding author: Nguyen Thai Anh}
}

\maketitle

\begin{abstract}
Tabular foundation models (TFMs) can classify the nodes of a graph without training on it, by reading node and neighborhood features as table rows next to labeled context rows.
Work in this line reports predictive performance, not conformal coverage or prediction-set size.
To our knowledge, we give the first reliability study of the setting, with TabICL as the TFM and half of each graph as labeled context.
As for any predictor fixed before calibration, a frozen in-context predictor makes split conformal prediction exactly valid in finite samples, with no training, validation fold, or tuning on the target graph.
An audit across ten graphs then shows that the training-free TabICL posterior has lower expected calibration error (ECE) than GCN with temperature scaling (GCN+TS) on nine of them.
Its mean ECE over the ten graphs is 0.019, about 35 percent below the 0.029 of GCN+TS.
We also introduce HG-DAPS, a training-free diffusion score whose homophily gate reads only the in-context labels, so the guarantee still holds.
Relative to adaptive prediction sets (APS), it reduces mean set size by 5.8 to 17.1 percent on six homophilous graphs and changes it by under 1 percent on four heterophilous ones.
On two binary, class-imbalanced graphs, a pre-registered trap case shows that gating on raw rather than adjusted homophily lowers coverage among low-homophily nodes by 0.27 and 0.12.
Marginal coverage stays at the nominal 0.90 and masks this drop.
\end{abstract}


\section{Introduction}
Tabular foundation models (TFMs) such as TabPFN~\cite{tabpfn} and TabICL~\cite{tabicl} are pretrained on synthetic tabular tasks and classify a new table by conditioning on its labeled rows, with no gradient updates.
Recent work applies them to node classification without any training on the target graph~\cite{tabgfm,g2tfm}.\footnote{Training-free: no parameter update on the target graph, with labeled context nodes. Prior work calls this setting zero-shot~\cite{tabgfm} or in-context learning~\cite{g2tfm}.}
Each node becomes a row of its own and its neighborhood features, and a frozen TFM reads the labeled context rows and predicts the remaining nodes.
On standard benchmarks this recipe is competitive with trained graph neural networks (GNNs), but reported evaluations focus on accuracy.

Accuracy alone is not sufficient for deployment.
A fraud-flagging system must know when to refer a node to a human, and that decision rests on the reliability of the model's confidence.
Conformal prediction has been used to measure this reliability for vision foundation models~\cite{cvfmconformal}, but not for TFMs on graphs.
The gap is sharper for a training-free model, since a reliability step that needs training or a validation fold would forfeit its main appeal.

Two aspects make the graph question non-routine.
First, split conformal prediction needs calibration and test nodes to be exchangeable.
For a trained GNN this holds when training, early stopping, and tuning never depend on the calibration split~\cite{daps,cfgnn}, a condition that has to be checked for each protocol.
Second, diffusing conformity scores along edges helps when neighbors tend to share labels (homophily), but on a heterophilous graph it mixes evidence across unlike classes.
Deciding whether to diffuse therefore needs a homophily estimate that does not depend on the calibration split.
The intuitive statistic, raw edge homophily, is inflated whenever the chance same-label rate is high, as with few or imbalanced classes~\cite{platonov}.
A gate built on it can be wrong exactly where caution matters.

Our contributions are as follows (pipeline in Fig.~\ref{fig:overview}).
\begin{itemize}
\item \textbf{Valid for free (Lemma~\ref{lem:free}).}
We observe that split conformal prediction is exactly valid for a frozen in-context predictor by the transductive exchangeability argument of~\cite{daps,cfgnn}, with no training protocol to audit and no validation fold.
\item \textbf{A pre-registered audit.}
On nine of ten graphs, training-free TabICL has lower expected calibration error (ECE) than temperature-scaled GCN, with no post-hoc step.
\item \textbf{HG-DAPS.}
We gate DAPS-style score diffusion~\cite{daps} by the adjusted homophily of the context labels alone, with no tuning.
Relative to APS, it shrinks mean set size on the six homophilous graphs and changes it by under 1\% on the four heterophilous ones.
\item \textbf{A trap case.}
On two binary, class-imbalanced graphs, gating on raw edge homophily lowers coverage among low-homophily nodes, a drop that marginal coverage hides.
\end{itemize}

% ---------------------------------------------------------------- Figure 1
\begin{figure*}[!htb]
\centering
\includegraphics[width=\textwidth, keepaspectratio]{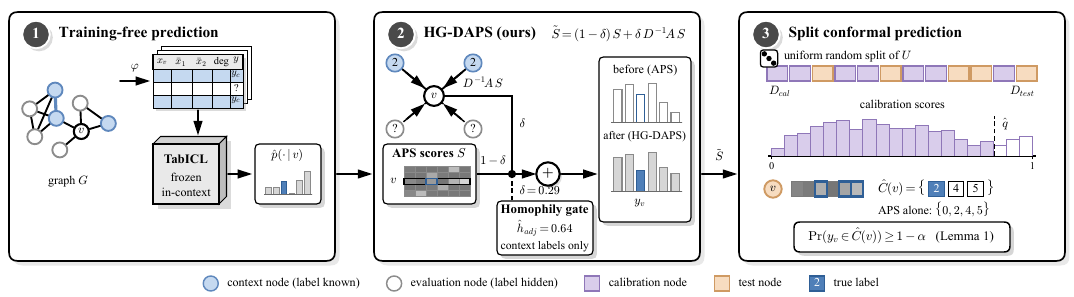}
\caption{Overview. (1)~A featurizer $\phi$ maps each node to a table row (reduced own features, one- and two-hop neighborhood means, one-hop extrema, log-degree), and a frozen TabICL reads the labeled context rows to return $\hat p(\cdot \mid v)$. (2)~HG-DAPS mixes the score matrix $S^{\mathrm{full}}$ (APS rows on $U$, one-hot LAC rows on $C$, written $S$ in the panel) with its row-normalized neighbor average as in \eqref{eq:diff}. The weight $\delta$ comes from a homophily gate that reads context labels only. On CiteSeer, $\hat h_{\mathrm{adj}}=0.64$ gives $\delta=0.29$, and the true-label score of the example node falls from $0.72$ to $0.55$. (3)~Only then is $U$ split at random into $\mathcal{D}_{\mathrm{cal}}$ and $\mathcal{D}_{\mathrm{test}}$, and $\hat C(v)=\{k:\tilde S_{v,k}\le\hat q\}$ keeps three labels where APS alone keeps four. Stages 1 and 2 do not depend on this split, so Lemma~\ref{lem:free} gives $\Pr(y_{v^\ast}\in\hat C(v^\ast))\ge 1-\alpha$ for a random test node $v^\ast$. Values are from one run (CiteSeer, context seed~0, $\alpha=0.10$, classes indexed from 0).}
\label{fig:overview}
\end{figure*}

% =====================================================================
\section{Related Work}
\label{sec:related}

\textbf{Tabular foundation models on graphs.}
Besides node-table methods~\cite{tabgfm,g2tfm}, GraphPFN pretrains a graph-specific prior-fitted network~\cite{graphpfn}.
These approaches reach competitive accuracy without task-specific training, but none reports conformal coverage or prediction-set quality.
Calibration appears only in a concurrent two-dataset ablation~\cite{logic2026}.

\textbf{Conformal prediction on graphs.}
Split conformal prediction gives distribution-free marginal coverage under exchangeability~\cite{vovk2005}, with adaptive prediction sets (APS)~\cite{aps}, regularized APS (RAPS)~\cite{raps}, and the least ambiguous set-valued classifier (LAC)~\cite{sadinle2019} as standard scores.
In the inductive setting, weighted conformal scores account for network dependence~\cite{clarkson2023distribution}.
In the transductive setting, exchangeability holds when the model uses no calibration or test labels and the split is uniform~\cite{daps,cfgnn}, and score diffusion preserves it~\cite{daps}.
DAPS (diffusion APS)~\cite{daps} diffuses nonconformity scores over edges, SNAPS~\cite{snaps} adds feature similarity and neighborhood information, and CF-GNN~\cite{cfgnn} learns a graph-based correction network.
HeAD-CP~\cite{headcp} derives node- and edge-wise diffusion coefficients from GNN softmax outputs, which a frozen TFM posterior could also supply.
HG-DAPS instead sets one graph-level diffusion weight from the adjusted homophily of the observed context labels, with no tuning.

\textbf{Calibration of graph models.}
Temperature scaling on held-out data is a standard remedy for miscalibrated neural networks~\cite{guo2017}.
CaGCN learns node-dependent calibration functions from neighboring confidence patterns~\cite{wang2021cagcn}, and GATS estimates node-wise temperatures with graph attention~\cite{hsu2022gats}.
Both need held-out labels and target trained GNNs, whereas we audit a frozen TabICL posterior with no post-hoc step.

% =====================================================================
\section{Setting and Exact Finite-Sample Validity}
\label{sec:lemma}

\subsection{Setting}
Let $G=(V,E)$ be an undirected graph with node features $X\in\mathbb{R}^{N\times d}$ and labels $y_v\in[K]=\{1,\dots,K\}$.
A random subset $C\subset V$ (the \emph{context}, drawn with a class-balance check, Section~\ref{sec:setup}) has observed labels $y_C$, and the remaining nodes $U=V\setminus C$ form the \emph{evaluation pool}.
Each node is mapped to a table row
\begin{align*}
\phi(v)=\big[\,&\mathrm{SVD}_{64}(x_v),\ \mathrm{SVD}_{64}(\bar x^{(1)}_v),\
\mathrm{SVD}_{64}(\bar x^{(2)}_v),\\
&\text{aggr},\ \log(1{+}\deg v)\,\big],
\end{align*}
where $\bar x^{(r)}_v$ is the mean of features over the $r$-hop neighborhood, aggr holds the one-hop maximum and minimum of the reduced features, and every singular value decomposition (SVD) is fit on all nodes using no labels.
A \emph{frozen} TFM $f$ reads a prompt of at most 5,000 context rows $(\phi(c),y_c)$ and returns by in-context learning, at softmax temperature $1$ and with no parameter update, a posterior $\hat p(\,\cdot\mid v)\in\Delta^{K-1}$ for every $v\in U$.

As nonconformity score we use APS~\cite{aps}.
With classes sorted by decreasing $\hat p(k\mid v)$, the score of label $k$ is the probability mass ranked before $k$ plus a randomized share of its own mass.
One uniform draw per row is made once and cached, so the score matrix $S\in[0,1]^{|U|\times K}$ of the evaluation pool is fixed before any split.
Given a calibration set $\mathcal{D}_{\mathrm{cal}}\subset U$ of size $n_{\mathrm{cal}}$ and a level $\alpha$, split conformal prediction takes the threshold
\begin{equation}
\hat q \;=\; s_{(j)},\qquad j=\lceil (n_{\mathrm{cal}}+1)(1-\alpha)\rceil,
\label{eq:qhat}
\end{equation}
the $j$-th smallest calibration score $s_i=S_{i,y_i}$ ($\hat q=\infty$ if $j>n_{\mathrm{cal}}$), and predicts $\hat C(v)=\{k: S_{v,k}\le \hat q\}$ for every test node.

\subsection{Valid for free}
Every component that produces the score matrix is \emph{fixed before any calibration data is chosen}: the graph, the features, the featurization, the context and its prompt, the frozen model with its inference seeds, and the APS draws.
Validity therefore needs no assumption about the data distribution, the model, or the graph, only a uniformly random calibration/test split of $U$.
Validity is thus free of training, of a validation fold, and of hyperparameter selection, though not free of labels, since the context and the calibration set are labeled.

\begin{lemma}[Exact finite-sample validity]
\label{lem:free}
Condition on $G$, $X$, all labels, the context $C$ with its prompt, the frozen predictor $f$ with its inference seeds, and the APS draws, so that the pairs $\{(S_{v,\cdot},\,y_v)\}_{v\in U}$ form a fixed finite population.
Let $\mathcal{D}_{\mathrm{cal}}$ be a uniformly random subset of $U$ of size $n_{\mathrm{cal}}$, and let the test node $v^\ast$ be drawn uniformly from $U\setminus\mathcal{D}_{\mathrm{cal}}$.
Then, with $\hat q$ as in \eqref{eq:qhat},
\[
\Pr\big(y_{v^\ast}\in \hat C(v^\ast)\big)\;\ge\;1-\alpha ,
\]
where the probability is over the random split alone.
\end{lemma}

\begin{proof}
The $n_{\mathrm{cal}}+1$ true-label scores $(s_1,\dots,s_{n_{\mathrm{cal}}},s_{v^\ast})$ are a uniform sample without replacement from the fixed population $\{S_{v,y_v}\}_{v\in U}$, so their joint law is invariant under permutation and they are exchangeable.
Break ties uniformly at random and let $R$ be the rank of $s_{v^\ast}$ among the $n_{\mathrm{cal}}+1$ scores.
By exchangeability $R$ is uniform on $\{1,\dots,n_{\mathrm{cal}}+1\}$.
If $R\le j$, at most $j-1$ calibration scores lie strictly below $s_{v^\ast}$, so $s_{v^\ast}\le s_{(j)}=\hat q$ and $y_{v^\ast}\in\hat C(v^\ast)$.
Hence the coverage probability is at least $\Pr(R\le j)=j/(n_{\mathrm{cal}}+1)\ge 1-\alpha$.
\end{proof}

Lemma~\ref{lem:free} is the finite-population form of Proposition~1 of~\cite{daps} and of the transductive exchangeability argument of~\cite{cfgnn}.
Because it conditions on the frozen predictor, it needs no permutation-equivariance assumption on the model.

\begin{remark}[No upper bound under ties]
\label{rem:ties}
The familiar upper bound $1-\alpha+\tfrac{1}{n_{\mathrm{cal}}+1}$ requires pairwise distinct true-label scores in $U$.
When the posterior takes few distinct values, the non-randomized LAC score $1-\hat p(y\mid v)$ has heavy ties.
This happens for the SVD-only ablation of TabICL on Minesweeper, whose raw features are one-hot counts of neighboring mines.
There LAC coverage reaches $0.96$ at $\alpha=0.10$.
This paper claims only the lower bound.
\end{remark}

\begin{remark}[Trained GNNs]
Lemma~\ref{lem:free} applies to any predictor that depends only on $(G,X,C,y_C)$ and on randomness independent of the calibration/test split, including a GNN trained and early-stopped on the context, as our baselines are.
For a trained model this premise has to be verified by auditing the training protocol, and a validation fold consumes context labels.
The training-free pipeline satisfies it with no training protocol to audit.
\end{remark}

% =====================================================================
\section{HG-DAPS: Homophily-Gated Diffusion Scores}
\label{sec:hgdaps}

% ---------------------------------------------------------- Table I
\begin{table*}[t]\centering\footnotesize
\caption{Training-free TabICL node classification at $\alpha=0.10$, mean$\pm$sd over 3 context seeds. Acc is TabICL accuracy. Coverage and set sizes are averaged over 20 calibration/test resplits per seed. ECE is compared with GCN+TS, the pre-registered comparator. Set sizes of APS and HG-DAPS share the same 0.90 coverage guarantee. Datasets are sorted by the context estimate $\hat h_{\mathrm{adj}}$.}
\label{tab:main}
\begin{tabular}{lrcccccc}\toprule
Dataset & $\hat h_{\mathrm{adj}}$ & Acc & ECE (TabICL) & ECE (GCN+TS) & Cov (APS) & Size (APS) & Size (HG-DAPS) \\ \midrule
Photo & 0.79 & 0.960{\scriptsize$\pm$0.004} & 0.007{\scriptsize$\pm$0.001} & 0.017{\scriptsize$\pm$0.004} & 0.901{\scriptsize$\pm$0.002} & 1.047{\scriptsize$\pm$0.009} & 0.956{\scriptsize$\pm$0.004} \\
Cora & 0.79 & 0.868{\scriptsize$\pm$0.005} & 0.028{\scriptsize$\pm$0.004} & 0.034{\scriptsize$\pm$0.004} & 0.901{\scriptsize$\pm$0.004} & 1.362{\scriptsize$\pm$0.022} & 1.151{\scriptsize$\pm$0.018} \\
Coauthor-CS & 0.78 & 0.934{\scriptsize$\pm$0.001} & 0.008{\scriptsize$\pm$0.004} & 0.007{\scriptsize$\pm$0.002} & 0.900{\scriptsize$\pm$0.001} & 1.101{\scriptsize$\pm$0.014} & 0.988{\scriptsize$\pm$0.004} \\
PubMed & 0.68 & 0.877{\scriptsize$\pm$0.005} & 0.010{\scriptsize$\pm$0.002} & 0.014{\scriptsize$\pm$0.002} & 0.900{\scriptsize$\pm$0.001} & 1.198{\scriptsize$\pm$0.009} & 1.128{\scriptsize$\pm$0.017} \\
CiteSeer & 0.64 & 0.757{\scriptsize$\pm$0.011} & 0.041{\scriptsize$\pm$0.003} & 0.054{\scriptsize$\pm$0.014} & 0.902{\scriptsize$\pm$0.003} & 1.921{\scriptsize$\pm$0.060} & 1.719{\scriptsize$\pm$0.063} \\
WikiCS & 0.59 & 0.860{\scriptsize$\pm$0.003} & 0.015{\scriptsize$\pm$0.000} & 0.040{\scriptsize$\pm$0.004} & 0.900{\scriptsize$\pm$0.002} & 1.482{\scriptsize$\pm$0.022} & 1.229{\scriptsize$\pm$0.002} \\
Amazon-Ratings & 0.14 & 0.485{\scriptsize$\pm$0.006} & 0.021{\scriptsize$\pm$0.002} & 0.023{\scriptsize$\pm$0.003} & 0.899{\scriptsize$\pm$0.003} & 2.875{\scriptsize$\pm$0.021} & 2.856{\scriptsize$\pm$0.010} \\
Tolokers & 0.10 & 0.827{\scriptsize$\pm$0.003} & 0.016{\scriptsize$\pm$0.003} & 0.046{\scriptsize$\pm$0.003} & 0.900{\scriptsize$\pm$0.003} & 1.298{\scriptsize$\pm$0.012} & 1.288{\scriptsize$\pm$0.010} \\
Minesweeper & 0.01 & 0.850{\scriptsize$\pm$0.005} & 0.010{\scriptsize$\pm$0.002} & 0.020{\scriptsize$\pm$0.004} & 0.899{\scriptsize$\pm$0.001} & 1.235{\scriptsize$\pm$0.003} & 1.234{\scriptsize$\pm$0.003} \\
Actor & 0.00 & 0.362{\scriptsize$\pm$0.007} & 0.031{\scriptsize$\pm$0.004} & 0.036{\scriptsize$\pm$0.008} & 0.901{\scriptsize$\pm$0.003} & 3.406{\scriptsize$\pm$0.021} & 3.406{\scriptsize$\pm$0.021} \\
\bottomrule\end{tabular}\end{table*}

Diffusing nonconformity scores over the graph sharpens prediction sets when neighbors tend to share labels.
On heterophilous graphs the same smoothing mixes scores of unlike classes and can enlarge the sets~\cite{headcp}.
DAPS selects its diffusion weight by tuning on labeled nodes~\cite{daps}, and HeAD-CP adapts node-wise coefficients from model predictions~\cite{headcp}.
HG-DAPS decides whether to diffuse from the observed context labels, so the decision does not depend on the calibration split.

\subsection{A context-only homophily gate}
Let $E_C\subseteq E$ be the set of \emph{intra-context} edges, whose endpoints both lie in $C$, and let $m=|E_C|$.
All quantities below use only $E_C$ and $y_C$.
The raw edge homophily $h_{\mathrm{edge}}$ is the fraction of edges in $E_C$ whose endpoints share a label.
We use the adjusted homophily of Platonov et al.~\cite{platonov},
\begin{equation}
\hat h_{\mathrm{adj}}
=\frac{h_{\mathrm{edge}}-P}{1-P},\qquad P=\sum_{k=1}^{K}p_k^{2},
\label{eq:hadj}
\end{equation}
where $p_k$ is the fraction of the $2m$ endpoints of $E_C$ that carry label $k$.
Here $P$ is the same-label rate expected when edge endpoints are matched at random with their labels.
So $\hat h_{\mathrm{adj}}$ is about $0$ when edges ignore labels, whatever the class balance, and it equals $1$ only when $h_{\mathrm{edge}}=1$.
Since $\hat h_{\mathrm{adj}}$ is a ratio estimate from $m$ edges, we shrink it toward $0$ when the context has few intra-context edges:
\begin{equation}
\delta \;=\; \delta_{\max}\,
\mathrm{clip}_{[0,1]}\!\Big(\hat h_{\mathrm{adj}}\cdot\frac{m}{m+m_0}\Big),
\label{eq:gate}
\end{equation}
with $\delta_{\max}=0.5$ and $m_0=100$ set a priori in the locked configuration.
We set $\delta=0$ when $m=0$ or $P=1$, and then, as whenever $\hat h_{\mathrm{adj}}\le 0$, HG-DAPS reduces exactly to APS.
Two ablations use the same form: the raw-homophily gate replaces $\hat h_{\mathrm{adj}}$ by $h_{\mathrm{edge}}$ in \eqref{eq:gate}, and ungated DAPS fixes $\delta=\delta_{\max}$.

\subsection{Gated diffusion of the score matrix}
Extend the score matrix to all of $V$ as $S^{\mathrm{full}}\in[0,1]^{|V|\times K}$.
A context row is the one-hot LAC score of its \emph{known} label, $S^{\mathrm{full}}_{c,k}=\mathbf{1}[k\neq y_c]$, so known labels enter the diffusion as certain evidence, and evaluation rows keep their cached APS scores.
With adjacency matrix $A$ and degree matrix $D$, one gated diffusion step gives
\begin{equation}
\tilde S \;=\;(1-\delta)\,S^{\mathrm{full}}
+\delta\,D^{-1}\!A\,S^{\mathrm{full}} ,
\label{eq:diff}
\end{equation}
and HG-DAPS calibrates the threshold \eqref{eq:qhat} on the true-label entries of $\tilde S$ in $U$ and forms prediction sets from the same matrix.
Nodes of degree zero keep their own row.
The sets remain nested in $\hat q$ but need not be prefixes of the TabICL ranking.
The cost is one sparse multiply, $O(|E|\,K)$, on top of one cached forward pass of the frozen model.

\begin{corollary}[Validity is preserved]
\label{cor:valid}
The map $S^{\mathrm{full}}\mapsto\tilde S$ is a fixed function of $(A,\delta)$, $S^{\mathrm{full}}$ depends only on $(S,y_C)$, and $\delta$ in \eqref{eq:gate} depends only on $(E_C,y_C)$.
All are determined before the calibration set is drawn, so $\{(\tilde S_{v,\cdot},y_v)\}_{v\in U}$ is again a fixed population and Lemma~\ref{lem:free} applies verbatim to HG-DAPS.
Coverage of at least $1-\alpha$ holds for every graph and every gate value, including a misestimated one.
\end{corollary}

Corollary~\ref{cor:valid} mirrors Proposition 2 of~\cite{daps} for a gated weight.
Only \emph{efficiency} (set size) and \emph{conditional} coverage across node groups depend on the gate.
This motivates the pre-registered trap-case hypothesis H4.
On Minesweeper and Tolokers, labels are binary and imbalanced, so $h_{\mathrm{edge}}$ is high while the mean $\hat h_{\mathrm{adj}}$ is at most 0.10.
A gate built on $h_{\mathrm{edge}}$ turns diffusion on and is expected to lower coverage for low-homophily nodes, whereas the adjusted gate \eqref{eq:gate} stays near zero.
Marginal coverage stays at least $1-\alpha$ in both cases, so an audit of marginal coverage alone would miss the failure.

\section{Experimental Setup}\label{sec:setup}

% ---------------------------------------------------------- Result figures (2, 3, 4), full width
\begin{figure*}[!t]
\centering
\includegraphics[width=0.80\textwidth]{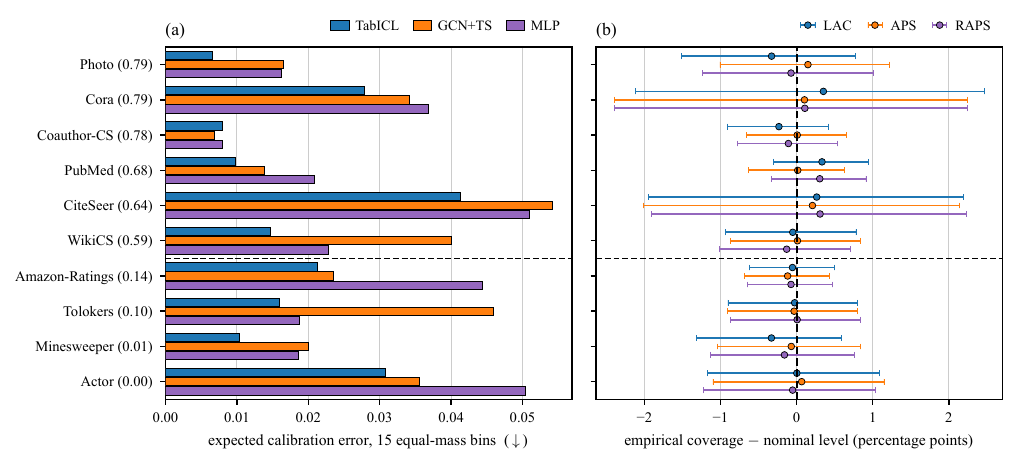}
\caption{Reliability audit at $\alpha=0.10$. Datasets are ordered by the context estimate $\hat h_{\mathrm{adj}}$ (in parentheses), and the dashed rule separates homophilous ($\hat{h}_{\mathrm{adj}}>0.5$) from heterophilous ($\hat{h}_{\mathrm{adj}}<0.2$) graphs. (a)~ECE with 15 equal-mass bins~\cite{nixon2019} for TabICL, GCN+TS, and MLP, with the TabICL and GCN+TS values in Table~\ref{tab:main}. (b)~Marginal coverage of LAC, APS, and RAPS on the TabICL posterior, averaged over 60 calibration/test resplits, as deviation from the nominal $0.90$. Error bars are exact 95\% Clopper-Pearson intervals~\cite{clopper1934} at each mean for $n_{\mathrm{test}}$ trials. They are sized for one test set, not for the 60-resplit mean, and every error bar covers the nominal level.}
\label{fig:audit}

\vspace{0.6em}

\includegraphics[width=\textwidth]{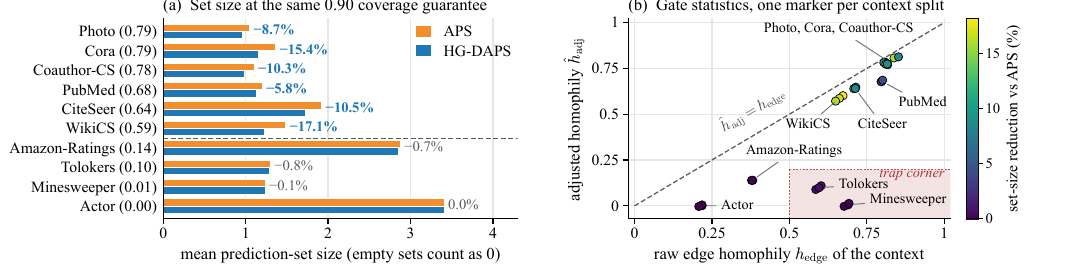}
\caption{HG-DAPS at $\alpha=0.10$, ordered as in Fig.~\ref{fig:audit}. (a)~Mean set size of APS and HG-DAPS at the same coverage guarantee, with the relative change. (b)~Gate statistics per context split, colored by set-size reduction. By \eqref{eq:hadj}, the gap below the diagonal $\hat h_{\mathrm{adj}}=h_{\mathrm{edge}}$ is $P(1-h_{\mathrm{edge}})/(1-P)$, which grows with class imbalance at fixed $h_{\mathrm{edge}}$. The shaded trap corner ($h_{\mathrm{edge}}>0.5$, $\hat h_{\mathrm{adj}}<0.2$) holds both trap datasets.}
\label{fig:method}

\vspace{0.6em}

\includegraphics[width=\textwidth]{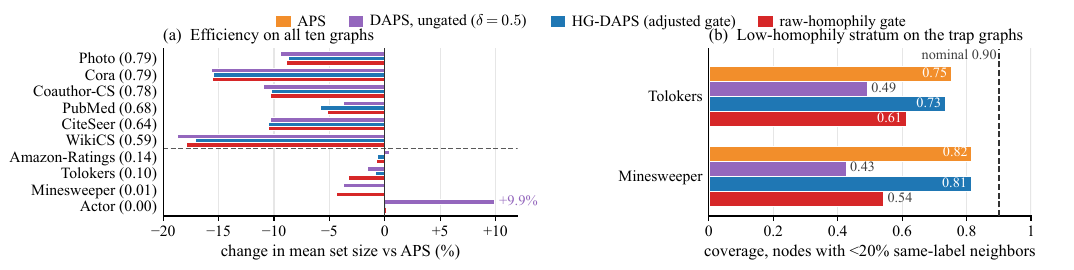}
\caption{Gate ablation at $\alpha=0.10$, ordered as in Fig.~\ref{fig:audit}. Ungated DAPS and the raw-homophily gate are defined after \eqref{eq:gate}. (a)~Change in mean set size relative to APS. (b)~Coverage in the low-homophily stratum (fewer than 20\% same-label neighbors) on the two trap datasets. Their mean marginal coverage over 60 resplits is 0.897 to 0.901 for all four methods, which hides the drop.}
\label{fig:ablation}
\end{figure*}

\textbf{Datasets.}
We use ten standard node classification benchmarks, with context estimates $\hat h_{\mathrm{adj}}$ from $0.00$ to $0.79$.
Cora, CiteSeer, and PubMed~\cite{sen2008}, Photo and Coauthor-CS~\cite{shchur2018}, and WikiCS~\cite{mernyei2020} are homophilous ($\hat h_{\mathrm{adj}}>0.5$).
Actor~\cite{pei2020}, Amazon-Ratings, Minesweeper, and Tolokers~\cite{platonovbench} are heterophilous ($\hat h_{\mathrm{adj}}<0.2$).

\textbf{Protocol.}
For each dataset we draw three $50\%$ context splits $C$, and the remaining nodes form $U$.
Each context receives 20 calibration/test resplits of $U$, giving 60 resplits per dataset, with $n_{\mathrm{cal}}=\min(1000,\lfloor|U|/2\rfloor)$ and $n_{\mathrm{test}}=|U|-n_{\mathrm{cal}}$.
All reported conformal results use $\alpha=0.10$.
The low-homophily stratum holds the evaluation nodes with fewer than 20\% same-label neighbors, a label-based quantity used for evaluation only.

\textbf{Models.}
The primary model is TabICL on the feature table $\phi$.
On Amazon-Ratings, Coauthor-CS, PubMed, Tolokers, and WikiCS, the context holds more than 5,000 nodes, so the TabICL prompt uses only part of the context labels.
A feature-only ablation (SVD-only) runs the same frozen TabICL on the SVD block of the raw node features alone.
The trained baselines are a graph convolutional network (GCN)~\cite{kipf2017}, a graph attention network (GAT)~\cite{velickovic2018}, and GCN with temperature scaling (GCN+TS), tuned on a validation fold carved from the context.
A multilayer perceptron (MLP) and logistic regression use the same feature table.
Label propagation~\cite{zhu2002} completes the eight models.

\textbf{Safeguards.}
Every random draw is seeded deterministically from the dataset name, the context seed, and the purpose of the draw.
The context split is redrawn when its class frequencies deviate by more than 0.05 from those of the full graph.
The validation fold and the prompt are checked in the same way against the context.
This check selects $C$ using labels outside it, but Lemma~\ref{lem:free} conditions on $C$, so validity is unaffected.
The calibration/test resplits are pure uniform draws, as the lemma requires.
The four hypotheses (H1 audit, H2 validity, H3 method, H4 trap case), their thresholds, and the decision rule were fixed in a locked configuration before the full run.
Inference runs once per context split on one 8\,GB consumer GPU, and all later CPU steps take seconds.

% =====================================================================

\section{Results}
\label{sec:results}

\noindent\hspace*{\parindent}\textbf{Audit (H1).}
In the pre-registered comparison, TabICL has lower ECE than GCN+TS on 9 of the 10 datasets (Table~\ref{tab:main}, Fig.~\ref{fig:audit}(a)).
The exception is Coauthor-CS, the only graph with more than ten classes, where GCN+TS reaches 0.007 against 0.008.
On Actor, Amazon-Ratings, and CiteSeer, the ECE advantage of TabICL is smaller than the standard deviation of GCN+TS over context seeds.
In a secondary comparison with all trained baselines, TabICL has the lowest ECE on 8 of the 10 datasets.
Its mean ECE over the ten datasets is 0.019, about 35\% below the 0.029 of GCN+TS and of MLP, with no task-specific training and no post-hoc calibration.
It also attains the highest mean accuracy, 0.778 against 0.752 for GCN.
On the binary Minesweeper and Tolokers, its accuracy is about 0.05 above the majority-class rate.
Its SVD-only ablation has lower ECE on three datasets, at lower mean accuracy (0.733).
Label propagation is poorly calibrated, with mean ECE 0.120.\\
\hspace*{\parindent}\textbf{Validity (H2).}
The pre-registered check covers 180 cells of the TabICL posterior: 10 datasets, 3 context seeds, and 6 scores (LAC, APS, RAPS, ungated DAPS, HG-DAPS, and the raw-homophily gate).
In every cell, the mean coverage over the 20 resplits of that seed lies inside the exact two-sided Clopper-Pearson 95\% interval at 0.90 for $n_{\mathrm{test}}$ trials.
The lowest cell mean is 0.894 (Minesweeper, LAC).
The band treats one test set as binomial and ignores calibration variability, so it is too strict for a single resplit and loose for a mean over 20 resplits.
It is therefore only a coarse sanity check, and its outcome is consistent with Lemma~\ref{lem:free} (pooled means in Fig.~\ref{fig:audit}(b)).
Outside the pre-registered cells, LAC on the SVD-only ablation exceeds the upper edge on Minesweeper, the tie-induced overcoverage of Remark~\ref{rem:ties}.\\
\hspace*{\parindent}\textbf{Method (H3).}
The pre-registered H3 contrast covers five homophilous graphs.
On all five, HG-DAPS reduces mean set size relative to APS by 5.8\% (PubMed) to 17.1\% (WikiCS), and it does so by 10.3\% on Coauthor-CS (Table~\ref{tab:main}, Fig.~\ref{fig:method}(a)).
On these six graphs its mean coverage over 60 resplits stays within 0.2 percentage points of that of APS.
On Photo and Coauthor-CS, accuracy above 0.90 forces APS, and a mean size below 1 forces HG-DAPS, to output some empty sets, which count as size 0 and may explain part of the reduction.
On the four heterophilous graphs the adjusted gate gives $\delta\le 0.07$ and mean set size changes by less than 1\%, while ungated diffusion enlarges Actor sets by 9.9\% (Fig.~\ref{fig:ablation}(a)).
On the six homophilous graphs, ungated DAPS shrinks the sets by a similar average amount (11.5\% against 11.3\%), so there the gate mainly decides whether to diffuse.\\
\hspace*{\parindent}\textbf{Trap case (H4).}
On Minesweeper and Tolokers the mean $h_{\mathrm{edge}}$ is 0.69 and 0.59 (Fig.~\ref{fig:method}(b)), so the raw gate diffuses with $\delta=0.34$ and $0.30$, whereas the adjusted gate gives $\delta=0.003$ and $0.05$.
Marginal coverage stays at the nominal 0.90, but coverage in the low-homophily stratum falls by 0.27 and 0.12 relative to the adjusted gate (Fig.~\ref{fig:ablation}(b)).
The raw gate also shrinks mean set size relative to APS by 4.4\% and 3.3\%, against 0.1\% and 0.8\% for HG-DAPS, so an efficiency audit would favor it.
All methods under-cover this stratum, with APS at 0.82 and 0.75, and the adjusted gate stays within 0.02 of APS, so it limits the harm without restoring coverage.\\
\hspace*{\parindent}\textbf{Scope and limitations.}
Lemma~\ref{lem:free} and Corollary~\ref{cor:valid} hold for any frozen predictor, but the ECE and set-size results are specific to TabICL.
We did not evaluate TabPFN v2, GraphPFN~\cite{graphpfn}, or HeAD-CP~\cite{headcp} on the same posterior.
The guarantee averages over calibration draws, so one fixed calibration set can realize coverage below $1-\alpha$.
Set sizes are compared with APS only, not with LAC, RAPS, or DAPS with a context-tuned weight, and empty-set rates are not reported.
Stratum coverage is reported only on the trap graphs, although HG-DAPS diffuses on homophilous graphs with weights such as $\delta=0.29$ on CiteSeer (Fig.~\ref{fig:overview}), close to the raw-gate weights of H4.
We did not vary $\delta_{\max}$, $m_0$, or the 50\% context fraction, and no graph has $\hat h_{\mathrm{adj}}$ between 0.2 and 0.5.
We also did not evaluate class-specific calibration~\cite{sadinle2019}, multi-hop diffusion, or a variant that separates the one-hot context rows, which DAPS does not use~\cite{daps}, from score smoothing.

% =====================================================================
\section{Conclusion}
\label{sec:conclusion}

A frozen in-context predictor makes split conformal prediction exactly valid in finite samples with no training protocol to audit (Lemma~\ref{lem:free}).
On ten graphs, training-free TabICL has lower ECE than GCN+TS on nine, and HG-DAPS reduces mean set size relative to APS on the six homophilous ones.
Its adjusted gate avoids the extra loss of 0.12 to 0.27 in low-homophily coverage that raw-homophily gating causes on the two trap graphs.
Whether HG-DAPS lowers that coverage inside homophilous graphs is the main open question.

\bibliographystyle{IEEEtran}
\bibliography{refs}

\end{document}